\documentclass[11pt]{article}

\usepackage[in]{fullpage}
\usepackage{graphicx,verbatim}
\usepackage[rflt]{floatflt}
\usepackage{epsfig,times,subfigure,graphicx}
\usepackage{amsmath,amssymb,amsopn,algorithm,algorithmic,theorem,float,bbm,bm,enumerate,color,multirow}
\usepackage{fullpage,setspace}
\usepackage{rotating}
\usepackage{array}
\usepackage{hyperref}
\usepackage{url}
\usepackage{changepage}
\usepackage{url}
\usepackage{chngcntr}
\usepackage{nameref}
\usepackage{zref-xr}
\zxrsetup{toltxlabel}
\usepackage{epstopdf}

\newcommand{\beq}{\begin{equation}}
\newcommand{\eeq}{\end{equation}}

\newcommand\I{\mathbb{I}}

\newcommand\s{\mathbb{S}}
\newcommand\R{\mathbb{R}}

\renewcommand\P{\mathbb{P}}

\renewcommand{\b}{\mathbf{b}}

\newcommand{\e}{\mathbf{e}}
\newcommand{\f}{\mathbf{f}}
\renewcommand{\r}{\mathbf{r}} 
\renewcommand{\u}{\mathbf{u}}
\renewcommand{\v}{\mathbf{v}}

\newcommand{\x}{\mathbf{x}}
\newcommand{\y}{\mathbf{y}}
\newcommand{\z}{\mathbf{z}}

\newcommand{\cA}{{\cal A}}
\newcommand{\cB}{{\cal B}}
\newcommand{\cC}{{\cal C}}

\newcommand{\cL}{{\cal L}}
\newcommand{\cM}{{\cal M}}
\newcommand{\cN}{{\cal N}}

\newcommand{\cS}{{\cal S}}

\newcommand{\bA}{\mathbf{A}}

\newcommand{\bX}{\mathbf{X}}

\newcommand{\bZ}{\mathbf{Z}}
\newcommand{\bT}{\mathbf{T}}

\newcommand{\vertiii}[1]{{\left\vert\kern-0.25ex\left\vert\kern-0.25ex\left\vert #1
    \right\vert\kern-0.25ex\right\vert\kern-0.25ex\right\vert}}

\newcommand{\E}{\mathbb{E}}

\newcommand{\Th}{\boldsymbol{\theta}}

\DeclareMathOperator{\argmax}{argmax}
\DeclareMathOperator{\argmin}{argmin}

\DeclareMathOperator{\rank}{rank}
\DeclareMathOperator{\diag}{diag}

\DeclareMathOperator{\cov}{Cov}

\newcounter{exampleI}
{\theorembodyfont{\rmfamily} \theoremstyle{plain} }

\newcounter{exampleII}
{\theorembodyfont{\rmfamily} \theoremstyle{plain} }

\newcounter{exampleIII}
{\theorembodyfont{\rmfamily} \theoremstyle{plain} }

{\theorembodyfont{\rmfamily} }
{\theorembodyfont{\rmfamily} \newtheorem{defn}{Definition}}
\newtheorem{theo}{Theorem}

\newtheorem{lemm}{Lemma}

\newcommand{\proof}{\noindent{\itshape Proof:}\hspace*{1em}}

\newcommand{\qed}{\nolinebreak[1]~~~\hspace*{\fill} \rule{5pt}{5pt}\vspace*{\parskip}\vspace*{1ex}}

\newcommand {\commentout}[1] {}

\def\ints{{{\rm Z} \kern -.35em {\rm Z} }}  
\def\smallints{{{\rm Z} \kern -.3em {\rm Z} }}  
\def\pints{{{\rm I} \kern -.15em {\rm N} }}      
\newcommand{\reals}{\mathbb R}

\def\cplx{{{\rm I} \kern -.45em {\rm C} }}       
\def\l2{\rm {\mathcal L}^{2}(\reals)}            

\newtheorem{nad}{Notation and Definitions}[section]

\newcommand{\be}{\begin{eqnarray}}
\newcommand{\ee}{\end{eqnarray}}
\newcommand{\bea}{\begin{eqnarray}}
\newcommand{\eea}{\end{eqnarray}}
\newcommand{\beaa}{\begin{eqnarray*}}
\newcommand{\eeaa}{\end{eqnarray*}}
\newcommand{\bnad}{\begin{nad}}
\newcommand{\enad}{\end{nad}}

\newcommand{\sign}{{\mbox{\rm sign}}}

\renewcommand{\widetilde}{\tilde}
\renewcommand{\widehat}{\hat}

\title{Sparse Linear Isotonic Models}
\date{\today}
\author{Sheng Chen \qquad \qquad Arindam Banerjee \vspace*{2mm}
\\
\{shengc,banerjee@cs.umn.edu\}\vspace*{2mm}\\
Department of Computer Science \& Engineering\\
University of Minnesota, Twin Cities}

\begin{document}

\maketitle

\begin{abstract}
In machine learning and data mining, linear models have been widely used to model the response as parametric linear functions of the predictors. To relax such stringent assumptions made by parametric linear models, additive models consider the response to be a summation of unknown transformations applied on the predictors; in particular,  additive isotonic models (AIMs) assume the unknown transformations to be monotone. In this paper, we introduce \emph{sparse linear isotonic models} (SLIMs) for high-dimensional problems by hybridizing ideas in parametric sparse linear models and AIMs, which enjoy a few appealing advantages over both. In the high-dimensional setting, a two-step algorithm is proposed for estimating the sparse parameters as well as the monotone functions over predictors. Under mild statistical assumptions, we show that the algorithm can accurately estimate the parameters. Promising preliminary experiments are presented to support the theoretical results.

\end{abstract}

\section{Introduction}
\label{sec:intro}
Linear models of the parametric form $y = \langle \tilde{\Th}, \x \rangle + \epsilon$ have enjoyed the enormous popularity in both machine learning and data mining communities for decades, due to its simplicity and interpretability. Here $y \in \R$ and $\x = [x_1, x_2, \ldots, x_p]^T \in \R^p$ are observed response and predictors respectively, and $\epsilon \in \R$ is a zero-mean additive noise. It is well known that the unknown parameter $\tilde{\Th} = [\tilde{\theta}_1, \tilde{\theta}_2, \ldots, \tilde{\theta}_p]^T$ can be estimated from multiple instances of $(\x, y)$, denoted by $\{(\x_i, y_i)\}_{i=1}^n$ , using least squares regression (typically when $n > p$). In recent years, linear models also prove to be able to survive the high-dimensional setting (i.e., $n \ll p$), by exploiting the sparsity of $\tilde{\Th}$ (i.e., $\tilde{\Th}$ has few nonzero entries). Under very mild statistical assumptions, various estimators have been proposed to find $\tilde{\Th}$ with provable guarantees on estimation error, such as Lasso \cite{tibs96,wain09} and Dantzig selector \cite{cata07,birt09}. Despite the prevalent success of linear models, modern data often arise from complex environments in which the linear correlation could break down, leading to poor performance of linear models. Progress has been made to relax the stringent assumption of linear models by allowing nonlinearity. In particular, \cite{bacc89} consider the following \emph{additive isotonic models} (AIMs),
\beq
\label{add_model}
y = \sum_{j=1}^p f_j(x_j) + \epsilon ~,
\eeq
where $\{f_j\}_{j=1}^p \triangleq \cal F$ is a set of \emph{monotone} univariate functions. To estimate $\cal F$, a commonly-used procedure is \emph{cyclic pooled adjacent violators} (CPAV). At each iteration of CPAV, \emph{isotonic regression} is called to estimate one $f_j$ and its solution can be efficiently found by the \emph{pooled adjacent violators algorithm} (PAVA) \cite{barl72}. Though the non-linearity can be captured by $\cal F$, one need to specify the monotonicity for each $f_j$ (either increasing or decreasing) in advance, which could be unknown in real-world applications, and enumerating all possible combinations can be computationally prohibitive. In high dimension, the estimation of $\cal F$ becomes even more challenging, because the number of monotone functions is very large.

To address the challenges in AIMs, we propose the \emph{sparse linear isotonic models} (SLIMs), which assume
\beq
\E \left[y | \x\right]  = \sum_{j = 1}^p \tilde{\theta}_j f_j(x_j) = \left\langle \tilde{\Th}, \f(\x) \right\rangle ~,
\eeq
where $\f(\x) \triangleq \tilde{\x} = [f_1(x_1), \ldots, f_p(x_p)]^T$. SLIMs combine the parametric form from the sparse linear models with the monotone transformations from AIMs, and generalize the assumption of additive noise $\epsilon$ to the conditional expectation form $\E[y |\x]$. Throughout the paper, the parameter $\tilde{\Th}$ is assumed to be \emph{$s$-sparse}. For identifiability, we also assume w.l.o.g. that each $f_j$ is monotonically increasing (as the monotonicity can be flipped by changing signs of $\tilde{\theta}_j$), and properly normalized such that every $\tilde{x}_j = f_j(x_j)$ is zero-mean and unit-variance. Note that without losing any representational power of AIM, the assumption of increasing $f_j$ avoids the pre-specification of monotonicity for each $f_j$ as required in \eqref{add_model}. For such hybrid model, given $n$ i.i.d. samples $\{(\x_i, y_i)\}_{i=1}^n$, our goal is to estimate both $\tilde{\Th}$ and $\cal F$. Since the hidden predictor $\tilde{\x}$ is inaccessible, brutally fitting data into a linear model could result in a poor estimate of $\tilde{\Th}$. In this work, we design a two-step algorithm to accomplish this goal, which estimates $\tilde{\Th}$ followed by $\cal F$. The estimation of $\tilde{\Th}$ is inspired by the \emph{rank-based} approaches for structure learning of graphical models. At the high level, those approaches do not rely on the exact values of samples generated from the graphical model, in order to learn its structure. Instead they resort to rank correlations (e.g., \emph{Kendall's tau correlation}  \cite{kend48}) that are invariant under monotonically increasing transformation, so that observing $\x$ and $\tilde{\x}$ makes no difference to the method. By leveraging a similar idea, we propose the Kendall's tau Dantzig selector (KDS) to estimate $\tilde{\Th}$, with certain Kendall's tau correlation coefficients appropriately plugged in. Under some distributional assumptions, we show that this estimator is guaranteed to recover a normalized version of $\tilde{\Th}$ with small error. After $\tilde{\Th}$ is estimated, we have a CPAV-type algorithm tailored for estimating transformations $\cal F$, which efficiently extends CPAV at little cost.

To sum up, we highlight a few merits of SLIM as follows. Firstly, as aforementioned, SLIM need not specify the monotonicity of $f_j$ whereas AIM requires.
Secondly the two-step estimation for SLIM is particularly useful in high-dimensional settings. The estimation of $\tilde{\Th}$ may identify many ``don't-care'' $f_j$'s as their corresponding $\tilde{\theta}_j$'s are zero, thus reducing the problem size of estimating $\cal F$. Besides, estimating $\tilde{\Th}$ will suffice if one only focuses on variable selection.
For the ease of exposition, we introduce a few notations which will be used in the rest of the paper. We let $\y = [y_1, y_2, \ldots, y_n]^T \in \R^n$ be the response vector, $\bX = [\x_1, \x_2, \ldots, \x_n]^T \in \R^{n \times p}$ be the observed design matrix , and denote its columns by $\x^{j} \in \R^n$. Similarly $\tilde{\bX}$, $\tilde{\x}_i$ and $\tilde{\x}^{j}$ will denote the hidden counterpart of $\bX$, $\x_i$ and $\x^{j}$. Matrix is bold capital, and the corresponding bold lowercase is reserved for its rows (columns) with suitable subscripts (superscripts), and its entries are plain lowercase with subscripts indexing both row and column. In general, vectors are bold lowercase while scalars are plain lowercase. For a matrix, $\|\cdot\|_{\max}$ denotes the value of the largest entry in magnitude. The rest of the paper is organized as follows: we first review the related work in Section \ref{sec:rel}, and then provide an overview of the two-step algorithm for SLIM in Section \ref{sec:alg}. Next we analyze the recovery of $\tilde{\Th}$ and present the algorithmic details for estimating $\cal F$ in Section \ref{sec:analysis}. In Section \ref{sec:exp}, we demonstrate the effectiveness of SLIM through experiments. Section \ref{sec:conc} is dedicated to the conclusion. 

\section{Related Work}
\label{sec:rel}
AIM was initially proposed in \cite{bacc89}. \cite{mayu07} established the asymptotic properties of the CPAV procedure. The high-dimensional counterpart of AIMs (i.e., assuming most of $f_j$'s are zero), Lasso ISO (LISO), was studied by \cite{fame12}, where a modified CPAV is used to achieved the sparsity in $\cal F$. \cite{chen09} considered a semiparametric additive isotonic model by introducing an additional parametric model into \eqref{add_model}. On the other hand, \cite{hati90} considered an additive model of the same form as \eqref{add_model} for general $\cal F$. With suitable smoothing operator on $f_j$'s, a coordinate descent procedure called \emph{backfitting} can be applied to estimating $\cal F$. In high-dimensional regime , \cite{rllw09} correspondingly investigated the sparse additive models (SpAMs), which is solved a backfitting algorithm with extra soft-thresholding steps. Many other efforts have been spent by relying on the smoothness of $f_j$'s, including \cite{lizh06}, \cite{mevb09}, \cite{homa04}, and etc.

The method we use to estimate $\tilde{\Th}$ is closely related to the high-dimensional structure learning of graphical models. For sparse Gaussian graphical model, \cite{mebu06} proposed a neighborhood selection procedure for estimating the graph structure, which iteratively regresses each variable against the rest via Lasso. The neighborhood Dantzig selector \cite{yuan10} shares the similar spirit with this approach, which switches Lasso to Dantzig selector. Recent progress has shown that these approaches continue to work for some non-Gaussian distributions, such as \emph{nonparanormal distribution} \cite{lilw09}, by using rank correlations to approximate the latent correlation matrix \cite{xuzo12,lhyl12}. Similar results have been further generalized to \emph{transelliptical distribution} \cite{lihz12,hali13,hali14}.

\section{Overview of Two-Step Algorithm}
\label{sec:alg}
In this section, we present an overview of the two-step algorithm for the estimation of SLIM, which first estimates $\tilde{\Th}$ and then $\cal F$. Detailed analyses are deferred to Section \ref{sec:analysis}.

For the estimation of $\tilde{\Th}$, if the hidden design matrix $\tilde{\bX}$ could be observed, Dantzig selector \cite{cata07} can be used to estimate $\tilde{\Th}$ as normal linear models,
\beq
\label{eq:orc_dantzig}
\hat{\Th}_{\text{orc}} = \ \underset{\Th \in \R^p}{\argmin} \  \|\Th\|_1 \ \ \ \  \text{s.t.} \ \ \ \  \left\| \frac{1}{n} \tilde{\bX}^T\left(\tilde{\bX} \Th - \y \right) \right\|_{\infty} \leq \gamma_n ~,
\eeq
where $\gamma_n$ is a tuning parameter. A key observation from \eqref{eq:orc_dantzig} is that instead of exactly knowing $\tilde{\bX}$ and $\y$, it is sufficient to be given the (approximate) value of $\frac{\tilde{\bX}^T \tilde{\bX}}{n}$ and $\frac{\tilde{\bX}^T \y}{n}$ in order for \eqref{eq:orc_dantzig} to work. Note that the quantity $\frac{\tilde{\bX}^T \tilde{\bX}}{n}$ and its expectation $\tilde{\mathbf{\Sigma}} = \E [\tilde{\x}  \tilde{\x}^T]$ also arise in the structure learning of nonparanormal graphical models. Specifically if $\tilde{\x}$ follows a multivariate Gaussian $\cN(\mathbf{0}, \tilde{\mathbf{\Sigma}})$, then the observed predictor $\x$, represented as $\f^{-1}(\tilde{\x}) \triangleq [f_1^{-1}(\tilde{x}_1), \ldots, f_1^{-1}(\tilde{x}_p)]^T$, is by definition a nonparanormal distribution $NPN(\tilde{\mathbf{\Sigma}}, \f^{-1})$, in which $\tilde{\mathbf{\Sigma}}$ is often called \emph{latent correlation matrix}. Simply speaking, the nonparanormal distribution models the random vector whose coordinates are element-wise monotone transformations of a Gaussian random vector. To estimate $\tilde{\mathbf{\Sigma}}$ without knowing $\f$ or $\f^{-1}$, Kendall's tau correlation coefficient \cite{kend48} plays a key role in rank-based methods. Given data $\bX = [x_{ij}] \in \R^{n \times p}$, we define the \emph{sample Kendall's tau correlation matrix} $\hat{\bT} = [\hat{t}_{ij}] \in \R^{p \times p}$  as
\begin{gather}
\label{eq:sample_kendall_mat}
\hat{t}_{ij}  = \sum_{1\leq k, k' \leq n}  \frac{\sign ((x_{ki} - x_{k'i})(x_{kj} - x_{k'j}))}{n(n-1)} ~,
\end{gather}
and its transformed version $\hat{\mathbf{\Sigma}} = [\hat{\sigma}_{ij}] \in \R^{p \times p}$,
\begin{gather}
\label{eq:trans_sample_kendall_mat}
\hat{\sigma}_{ij} = \sin \left( \frac{\pi}{2} \hat{t}_{ij} \right), \ \ \ \
\end{gather}
\begin{algorithm}[t]
\renewcommand{\algorithmicrequire}{\textbf{Input:}}
\renewcommand{\algorithmicensure} {\textbf{Output:}}
\caption{Estimating $\tilde{\Th}$ and $\cal F$ for SLIM}
\label{alg:overview}
\begin{algorithmic}[1]
\REQUIRE $\bX \in \R^{n \times p}$, $\y \in \R^n$,  tuning parameter $\gamma_n$  \ \\
\ENSURE Estimated $\hat{\Th}$ for $\tilde{\Th}$ and $\hat{\cal F}$ for $\cal F$ \\
\STATE Compute the transformed sample Kendall's tau correlation matrix $\hat{\mathbf{\Sigma}}$ and vector $\hat{\boldsymbol{\beta}}$ using \eqref{eq:sample_kendall_mat} - \eqref{eq:trans_sample_kendall_vec}
\STATE Estimate $\check{\Th}$ via Kendall's tau Dantzig selector \eqref{eq:kds}
\STATE $\hat{\Th} := \hat{\sigma}_y \check{\Th}$, where $\hat{\sigma}_y$ is the sample variance of $\y$
\STATE Estimate the hidden design $\hat{\bX}$ via \eqref{eq:lsq_nrm}
\STATE $\hat{\cal F} := \{ \hat{f}_j \}_{j=1}^p$, where $\hat{f}_j$ is given by \eqref{eq:interpolate}
\STATE \textbf{Return} \ $\hat{\Th}$ and $\hat{\cal F}$
\end{algorithmic}
\end{algorithm}
One straightforward yet critical property of $\hat{\bT}$ and $\hat{\mathbf{\Sigma}}$ is the invariance to monotone increasing transformations on columns of $\bX$, indicating that the two quantities remain unchanged if $\bX$ is replaced by $\tilde{\bX}$ in the definitions. More importantly, later analysis will reveal for the class of transelliptical distributions (a generalization of nonparanormal distribution) the closeness between the transformed sample Kendall's tau correlation matrix $\hat{\mathbf{\Sigma}}$ and the latent correlation matrix $\tilde{\mathbf{\Sigma}}$, thus $\hat{\mathbf{\Sigma}}$ can serve as an approximation to $\frac{\tilde{\bX}^T \tilde{\bX}}{n}$ as $\frac{\tilde{\bX}^T \tilde{\bX}}{n} \approx \tilde{\mathbf{\Sigma}}$ in expectation. For $\frac{\tilde{\x}^T \y}{n}$ and its expectation $\tilde{\boldsymbol{\beta}} = \E[y\tilde{\x}] = \tilde{\mathbf{\Sigma}} \tilde{\Th}$, we similarly define the \emph{sample Kendall's tau correlation vector} $\hat{\b} \in \R^p$ and its transformation $\hat{\boldsymbol{\beta}}$
\begin{gather}
\label{eq:sample_kendall_vec}
\hat{b}_j = \sum_{1\leq k, k' \leq n} \frac{\sign ((x_{kj} - x_{k'j})(y_{k} - y_{k'}))}{n(n-1)} ~, \\
\label{eq:trans_sample_kendall_vec}
\hat{\beta}_j = \sin \left( \frac{\pi}{2} \hat{b}_{j} \right) ~,
\end{gather}
and use $\hat{\boldsymbol{\beta}}$ as a replacement for $\frac{\tilde{\x}^T \y}{n}$. Therefore the estimation of $\tilde{\Th}$ can proceed with \eqref{eq:orc_dantzig} by replacing  $\frac{\tilde{\bX}^T \tilde{\bX}}{n}$ and $\frac{\tilde{\bX}^T \y}{n}$ with $\hat{\mathbf{\Sigma}}$ and $\hat{\boldsymbol{\beta}}$ respectively, which leads to the following estimator which we call \emph{Kendall's tau Dantzig selector} (KDS),
\beq
\label{eq:kds}
\check{\Th} = \ \underset{\Th \in \R^p}{\argmin} \ \ \|\Th\|_1 \ \ \ \text{s.t.} \ \ \ \left\| \widehat{\mathbf{\Sigma}} \Th - \hat{\boldsymbol{\beta}} \right\|_{\infty} \leq \gamma_n ~.
\eeq
But it will be shown later in the analysis that the $\check{\Th}$ only approximates the direction of $\tilde{\Th}$, and the scale should be attached on the final estimate $\hat{\Th}$ by calculating the sample variance of $\y$.

To estimate the transformations $\cal F$, one needs to first find out an $\hat{\bX} = [\hat{x}_{ij}]$ that approximates the hidden design $\tilde{\bX} = [f_j(x_{ij})]$ for the observed $\bX = [x_{ij}]$, which essentially gives us the estimated values of each $f_j$ at $n$ points $x_{1j}, \ldots, x_{nj}$. To be specific, we fit $\hat{\bX}$ into $\y$ and the estimated $\hat{\Th}$ through the following convex program,
\beq
\begin{gathered}
\label{eq:lsq_nrm}
\hat{\bX} \ = \ \underset{\bZ \in \R^{n \times p}}{\argmin} \ \ \frac{1}{2} \|\bZ \hat{\Th} - \y\|_2^2 \ \ \ \ \text{s.t.} \ \ \ \ \z^j \in \cM(\x^j), \ \mathbf{1}^T \z^j = 0, \ \|\z^j\|_2 \leq \sqrt{n} ~,
\end{gathered}
\eeq
where $\cM(\x) = \{ \v \ | \ v_i \geq v_j \ \text{iff} \ x_i \geq x_j, \ \forall \ i,j \}$. In order to get the $f_j$ defined everywhere, we need to interpolate the $n$ estimated points $\hat{x}_{1j}, \ldots, \hat{x}_{nj}$. In the algorithm, we simply use nearest-neighbor interpolation as follows,
\beq
\label{eq:interpolate}
\hat{f}_j(x) = \sum_{i=1}^n \hat{x}_{ij} \cdot  \I \Big \{ i = \underset{1\leq k \leq n}{\argmin} \ |x_{kj} - x| \Big \} ~,
\eeq
where $\I\{\cdot\}$ is the indicator function that outputs one if the predicate is true and zero otherwise. Other interpolation technique, e.g., linear/spline interpolation, can be applied in the need of certain desired properties of $f_j$. The full estimation algorithm is given in Algorithm \ref{alg:overview}.

\section{Theoretical Analysis}
\label{sec:analysis}
In this section, we detail the Algorithm \ref{alg:overview} in several aspects. We analyze the recovery guarantee of the Kendall's tau Dantzig selector $\check{\Th}$ for estimating $\tilde{\Th}$. Under suitable assumptions on the distribution of $\x$, $y$, we show that the sample complexity analysis can be sharpened compared to earlier related work \cite{xuzo12,lhyl12}. For estimating $\cal F$, we propose a backfitting algorithm similar to CPAV, where each step can be solved nearly at the same cost as isotonic regression.

\subsection{Estimating $\tilde{\Th}$}
In this subsection, we consider the estimation of $\tilde{\Th}$. The KDS \eqref{eq:kds} can be casted as a linear program, which can be solved efficiently by many optimization algorithms \cite{chcb14,lebb16}. Hence we focus on the statistical aspect of KDS. From Section \ref{sec:alg}, we know that the success of KDS relies on $\hat{\mathbf{\Sigma}}$ and $\hat{\boldsymbol{\beta}}$, which replace $\frac{\tilde{\bX}^T \widetilde{\bX}}{n}$ and $\frac{\widetilde{\bX}^T \y}{n}$ in the Dantzig selector \eqref{eq:orc_dantzig}. Hence we first investigate the property of $\hat{\mathbf{\Sigma}}$ and $\hat{\boldsymbol{\beta}}$.
The definition \eqref{eq:sample_kendall_mat} - \eqref{eq:trans_sample_kendall_vec} are sample versions of (transformed) Kendall's tau correlation matrix and vector. Here we define their population counterparts.
\begin{defn}
Given $(\x, y)$ and its independent copy $(\x', y')$, the population Kendall's tau correlation matrix $\bT = [t_{ij}] \in \R^{p \times p}$ and vector $\b \in \R^p$ are defined as
\begin{align}
\label{pop_kendall_mat}
&t_{ij}  = \P \left( (x_i - x'_i)(x_j - x'_j) > 0 \right)  - \P \left( (x_i - x'_i)(x_j - x'_j) < 0 \right) , \\
&b_j = \P \left( (x_j - x'_j)(y - y') > 0 \right) - \P \left( (x_j - x'_j)(y - y') < 0 \right) ,
\end{align}
and their transformed versions $\mathbf{\Sigma} = [\sigma_{ij}] \in \R^{p \times p}$ and $\boldsymbol{\beta} \in \R^p$ are given by
\begin{gather}
\sigma_{ij} = \sin \left( \frac{\pi}{2} t_{ij} \right) ~, \\
\beta_{j} = \sin \left(\frac{\pi}{2} b_j \right)~.
\end{gather}
\end{defn}
Then we introduce two family of distributions, \emph{elliptical} and \emph{transelliptical}. The transelliptical distribution is defined based on the elliptical distribution given as follows.
\begin{defn}[Elliptical distribution]
A random vector $\z \in \R^p$ follows an elliptical distribution $EC(\boldsymbol{\mu}, \tilde{\mathbf{\Sigma}}, \xi)$ iff $\z$ has a stochastic representation: $\z \sim \boldsymbol{\mu} + \xi \bA \u$. Here $\boldsymbol{\mu} \in \R^p$, $q \triangleq \rank(\bA)$, $\bA \in \R^{p \times q}$, $\xi \geq 0$ is a random variable independent of $\u$, $\u \in \s^{q-1}$ is uniformly distributed on the unit sphere in $\R^q$, and $\bA \bA^T = \tilde{\mathbf{\Sigma}}$. Note that
\beq
\label{eq:elliptical_prop}
\E[\z] = \boldsymbol{\mu} ~, \ \ \ \ \ \ \cov[\z] = \frac{\E[\xi^2]}{q} \tilde{\mathbf{\Sigma}} ~.
\eeq
\end{defn}
This family of distribution contains the Gaussian distribution as a special case, and more details can be found in \cite{fakn90}. The extension from elliptical to transelliptical distribution parallels that from normal to nonparanormal distribution.
\begin{defn}[Transelliptical distribution]
A random vector $\x \in \R^p$ follows the transelliptical distribution $TE(\tilde{\mathbf{\Sigma}}, \xi, \f)$ if $\f(\x) = [f_1(x_1), f_2(x_2), \ldots, f_p(x_p)]^T \sim EC(\boldsymbol{\mu}, \tilde{\mathbf{\Sigma}}, \xi)$, where $f_1, f_2, \ldots f_p$ are all strictly increasing functions, $\boldsymbol{\mu} = \mathbf{0}$, $\diag(\tilde{\mathbf{\Sigma}}) = \mathbf{I}$, and $\P(\xi = 0) = 0$.
\end{defn}
The conditions on $\boldsymbol{\mu}$ and $\diag(\tilde{\mathbf{\Sigma}})$ are imposed for identifiability. If the underlying elliptical distribution is multivariate Gaussian, then the transelliptical family is reduced to the nonparanormal. We refer the readers to \cite{lihz12} for more discussions on transelliptical distribution. Based on the elliptical and transelliptical family, we introduce our assumptions on distribution of $(\x, y)$:
\begin{adjustwidth}{5mm}{}
\vspace{+2mm}
(\textbf{A1}) \quad $\x \in \R^p$ follows a transelliptical distribution $TE(\tilde{\mathbf{\Sigma}}, \xi, \f)$ (or equivalently $\tilde{\x} = \f(\x)$ follows a elliptical distribution $EC(\mathbf{0}, \tilde{\mathbf{\Sigma}}, \xi)$), and $\E [\xi^2] = p$. 

(\textbf{A2}) \quad The smallest eigenvalue $\lambda_{\min}$ of $\tilde{\mathbf{\Sigma}}$ is strictly positive, i.e., $\tilde{\x}$ is nondegenerate.

(\textbf{A3}) \quad $\tilde{\x}$ and $y$ are jointly elliptically distributed.
\vspace{+2mm}
\end{adjustwidth}
The assumption $\E [\xi^2] = p$ is also out of the consideration of identifiability. The assumption (\textbf{A3}) on the joint distribution of $(\tilde{\x}, y)$ may seem obscure. But it can be satisfied, for example, when $\x$ is nonparanormal and $y$ is a noisy observation of $\langle \tilde{\Th}, \f(\x) \rangle$ perturbed by an additive zero-mean Gaussian noise. Under these assumptions, one notable result that has been shown for $\mathbf{\Sigma}$, $\hat{\mathbf{\Sigma}}$ and $\tilde{\mathbf{\Sigma}}$ is given in the following lemma.

\begin{lemm}
\label{lem:kendall_mat}
For $\x \sim TE(\tilde{\mathbf{\Sigma}}, \xi, \f)$, the transformed population Kendall's tau correlation matrix $\mathbf{\Sigma}$ satisfies
\beq
\mathbf{\Sigma} = \tilde{\mathbf{\Sigma}} ~,
\eeq
and the sample version $\hat{\mathbf{\Sigma}}$ for $\mathbf{\Sigma}$ defined in \eqref{eq:trans_sample_kendall_mat}, with probability at least $1 - p^{-2.5}$, satisfies
\beq
\|\widehat{\mathbf{\Sigma}} - \tilde{\mathbf{\Sigma}}\|_{\max} \leq 3 \pi \sqrt{\frac{\log p}{n}}
\eeq
\end{lemm}
The lemma is essentially Theorem 3.2 and 4.1 in \cite{hali14}. Similarly we have the following lemma for $\boldsymbol{\beta}$, $\hat{\boldsymbol{\beta}}$ and $\tilde{\boldsymbol{\beta}}$.
\begin{lemm}
\label{lem:kendall_vec}
Under assumptions (\textbf{A1}) - (\textbf{A3}) for SLIMs, the transformed population Kendall's tau correlation vector $\boldsymbol{\beta}$ satisfies
\beq
\label{pop_kendall_vec}
\boldsymbol{\beta} = \frac{\tilde{\boldsymbol{\beta}}}{\sigma_y} = \frac{\tilde{\mathbf{\Sigma}} \tilde{\Th}}{\sigma_y}  ~,
\eeq
where $\sigma_y^2$ is the variance of $y$. The transformed sample Kendall's tau correlation vector $\hat{\boldsymbol{\beta}}$, with probability at least $1 - \frac{2}{p}$, satisfies
\beq
\|\hat{\boldsymbol{\beta}} - \boldsymbol{\beta}\|_{\infty} \leq 2 \pi \sqrt{\frac{\log p}{n}}
\eeq
\end{lemm}
\proof Given that $\lambda_{\min} > 0$ and \eqref{eq:elliptical_prop}, we have $\E[\tilde{\x}] = \mathbf{0}$, $\rank(\mathbf{A}) = \rank(\tilde{\mathbf{\Sigma}}) = p$ and $\cov[\tilde{\x}] = \tilde{\mathbf{\Sigma}}$. Since $\tilde{\x}$, $y$ are jointly elliptical and $\boldsymbol{\beta}$ is invariant to $\f$, using Theorem 2 in \cite{lims03}, we have for each $\beta_j$,
\begin{align*}
\beta_j = \frac{\E [y \tilde{x}_j] - \E[y] \E[\tilde{x}_j]}{\sqrt{\text{Var} [y]} \sqrt{\text{Var}  [\tilde{x}_j]}} = \frac{\E \left[ \langle \tilde{\Th}, \tilde{\x} \rangle \cdot \tilde{x}_j \right]}{\sqrt{\text{Var} [y]}}  = \frac{\langle \tilde{\Th}, \tilde{\boldsymbol{\sigma}}_{j} \rangle}{\sigma_y} ~,
\end{align*}
which implies \eqref{pop_kendall_vec}. Using Hoeffding's inequality for U-statistics, we have for each $\beta_j$ and $\hat{\beta}_j$
\begin{align*}
\P \left( \left| \beta_j - \hat{\beta}_j \right| \geq \epsilon \right) \leq 2 \exp \left( -\frac{n \epsilon^2}{2 \pi^2} \right) ~.
\end{align*}
Letting $\epsilon = 2 \pi \sqrt{\frac{\log p}{n}}$ and taking union bound, we obtain
\begin{align*}
\P \left( \left\| \boldsymbol{\beta} - \hat{\boldsymbol{\beta}} \right\|_{\infty} \geq  2 \pi \sqrt{\frac{\log p}{n}} \right) \leq \frac{2}{p} ~,
\end{align*}
which completes the proof. \qed

In the light of Lemma \ref{lem:kendall_mat} and \ref{lem:kendall_vec}, it becomes clear that $\frac{\widetilde{\bX}^T \widetilde{\bX}}{n}$ and $\frac{\widetilde{\bX}^T \y}{n}$ in \eqref{eq:orc_dantzig} are replaced by $\widehat{\mathbf{\Sigma}}$ and $\hat{\boldsymbol{\beta}}$ in \eqref{eq:kds}. The population counterpart of $\widehat{\mathbf{\Sigma}}$ is $\mathbf{\Sigma} = \tilde{\mathbf{\Sigma}} = \E [\frac{\widetilde{\bX}^T \widetilde{\bX}}{n}]$. Unfortunately, the population version $\boldsymbol{\beta}$ of $\hat{\boldsymbol{\beta}}$ is not equal to $\tilde{\boldsymbol{\beta}} = \E [\frac{\widetilde{\bX}^T \y}{n} ]$, which is additionally normalized by $\sigma_y$. Therefore we will see later that KDS recovers a scaled $\tilde{\Th}$.

In order to bound the estimation error, first we show that the transformed sample Kendall's tau correlation matrix $\widehat{\mathbf{\Sigma}}$ satisfies the \emph{restricted eigenvalue} (RE) condition \cite{birt09,zhou09,rawy10,nrwy12,bcfs14}, which is critical in the recovery analysis.
\begin{lemm}
\label{kendall_re}
Define the descent cone for any $s$-sparse vector $\Th^* \in \R^p$,
\beq
\label{des_cone}
\cC = \left\{ \v \in \R^p \ | \ \ \|\Th^* + \v \|_1 \leq \|\Th^*\|_1 \right\} ~.
\eeq
If $\x \sim TE(\tilde{\mathbf{\Sigma}}, \xi, \f)$ and $n \geq \left(\frac{24 \pi }{\lambda_{\min}}\right)^2 s^2 \log p = O\left(s^2 \log p\right)$, with probability at least $1 - p^{-2.5}$, the following RE condition holds for $\widehat{\mathbf{\Sigma}}$ in $\cC$,
\beq
\inf_{\v \in \cC \cap \s^{p-1}} \v^T \widehat{\mathbf{\Sigma}} \v \geq  \frac{\lambda_{\min}}{2} ~,
\eeq
where $\lambda_{\min}$ is the smallest eigenvalue of $\tilde{\mathbf{\Sigma}}$.
\end{lemm}
\begin{proof}
Let $\cS$ be the support of $\Th^*$, then we have
\begin{gather*}
\v \in \cC \cap \s^{p-1}  \ \ \Longrightarrow \ \ \|\Th^*_{\cS} + \v_{\cS} \|_1 + \|\v_{\cS^c}\|_1 \leq \|\Th^*\|_1 \\
\Longrightarrow \ \  \|\Th^*_{\cS}\|_1 - \|\v_{\cS} \|_1 + \|\v_{\cS^c}\|_1 \leq \|\Th^*\|_1 \ \ \Longrightarrow \\
\|\v_{\cS^c}\|_1 \leq \|\v_{\cS}\|_1  \Longrightarrow  \|\v\|_1 \leq 2 \|\v_{\cS}\|_1 \leq 2 \sqrt{s} \|\v_{\cS}\|_2 \leq 2 \sqrt{s}
\end{gather*}
With probability at least $1 - p^{-2.5}$, we have for any $\v \in \cC \cap \s^{p-1}$
\begin{align*}
\v^T \widehat{\mathbf{\Sigma}} \v &\geq \v^T \tilde{\mathbf{\Sigma}} \v - \left| \v^T \left(\widehat{\mathbf{\Sigma}} - \tilde{\mathbf{\Sigma}}\right) \v \right| \geq \lambda_{\min} - \left| \sum_{1 \leq i, j \leq p} v_i v_j \left(\hat{\sigma}_{ij} - \tilde{\sigma}_{ij} \right) \right| \\
&\geq \lambda_{\min} - \|\v\|_1^2 \left\| \widehat{\mathbf{\Sigma}} - \tilde{\mathbf{\Sigma}} \right\|_{\max} \geq \lambda_{\min} - 12 \pi  \sqrt{\frac{s^2 \log p}{n}} ~,
\end{align*}
where we use Lemma \ref{lem:kendall_mat} and the fact that $\|\v\|_1 \leq 2 \sqrt{s}$. Since we choose $n \geq \left(\frac{24 \pi }{\lambda_{\min}}\right)^2 s^2 \log p$, we have
\begin{align*}
\v^T \widehat{\mathbf{\Sigma}} \v \geq \lambda_{\min} - 12 \pi  \sqrt{\frac{s^2 \log p}{n}} \geq \lambda_{\min} - \frac{\lambda_{\min}}{2} = \frac{\lambda_{\min}}{2} ~,
\end{align*}
which completes the proof. \qed
\end{proof}

\textbf{Remark:} Similar proof steps appear in \cite{xuzo12} in the context of the analysis of rank-based neighborhood Dantzig selector, but the concept of the RE condition is not explicitly formulated. Later we will show a sharper sample complexity for RE condition in Theorem \ref{kendall_re_sharp}. Hence we single out Lemma \ref{kendall_re} here in order for a comparison.

From the analysis above, we see that the $O\left(s^2 \log p\right)$ sample complexity for RE condition of $\widehat{\mathbf{\Sigma}}$ is worse than that of $\frac{\tilde{\bX}^T \tilde{\bX}}{n}$, which is $O\left(s \log p\right)$ \cite{birt09,nrwy12}. Next we show that this sharper bound (see Theorem \ref{kendall_re_sharp}) can be obtained for $\widehat{\mathbf{\Sigma}}$ if the distribution of $\x$ further satisfies the \emph{sign sub-Gaussian condition} \cite{hali13}. This result may be of independent interest.

\begin{defn}[sign sub-Gaussian condition]
For a random variable $x$, the operator $\psi: \R \mapsto \R$ is defined as
\beq
\psi (x; \alpha, t_0) \triangleq \inf \left\{ c > 0: \E \exp \{t(x^\alpha - \E x^\alpha)\} \leq \exp(c t^2), \ \text{for} \ |t| < t_0 \right\} ~.
\eeq
The random vector $\x \in \R^p$ satisfies the sign sub-Gaussian condition iff
\beq
\sup_{\v \in \s^{p-1}} \psi \left( \left\langle \sign(\x - \x'), \v \right\rangle ; 2, t_0 \right) \leq \kappa \|\bT\|_2^2 ~,
\eeq
for a fixed constant $\kappa$ and a positive number $t_0 > 0$ such that $t_0 \kappa \|\bT\|_2^2$ is lower bounded by a fixed constant, where $\x'$ is an independent copy of $\x$ and $\bT$ is the population Kendall's tau correlation matrix defined in \eqref{pop_kendall_mat}.
\end{defn}
Detailed discussions on the sign sub-Gaussian condition can be found in \cite{hali13}, which is out of the scope of this paper. In particular, \cite{hali13} show that if sign sub-Gaussian condition for transelliptical $\x$, the $\widehat{\mathbf{\Sigma}}$ will converge with high probability to $\tilde{\mathbf{\Sigma}}$ at rate $O\left(\sqrt{\frac{s\log p}{n}}\right)$ in terms of \emph{restricted spectral norm},
\beq
\|\widehat{\mathbf{\Sigma}} - \tilde{\mathbf{\Sigma}}\|_{2,s} \triangleq \sup_{\substack{\v \in \s^{p-1} \\ \|\v\|_0 \leq s}} \left|\v^T (\widehat{\mathbf{\Sigma}} - \tilde{\mathbf{\Sigma}}) \v\right| = O\left(\sqrt{\frac{s \log p}{n}}\right) ~.
\eeq
Starting from this result, we show that with high probability the RE condition will hold for $\widehat{\mathbf{\Sigma}}$ with $O(s \log p)$ samples.
\begin{theo}
\label{kendall_re_sharp}
Let $\bX = [\x_1, \x_2, \ldots, \x_n]^T$ be i.i.d. samples of $\x \sim TE(\tilde{\mathbf{\Sigma}}, \xi, \f)$ for which the sign sub-Gaussian condition holds with constant $\kappa$. Define the constant
\begin{align*}
c_0 = \max \left\{\frac{320 \kappa \pi^4 \|\tilde{\mathbf{\Sigma}}\|_2^2}{\lambda^2_{\min}}, \frac{\pi^2}{\lambda_{\min}} \right\} ~,
\end{align*}
in which $\|\cdot\|_2$ denotes the spectral norm (i.e. the largest eigenvalue) and $\lambda_{\min}$ is the smallest eigenvalue $\tilde{\mathbf{\Sigma}}$. If $n \geq  \frac{128 c_0}{\lambda_{\min}} s \log p = O(s \log p)$, with probability at least $1 - \frac{2}{p} - \frac{1}{p^2}$, $\widehat{\mathbf{\Sigma}}$ satisfies the following RE condition,
\beq
\inf_{\v \in \cC \cap \s^{p-1}} \v^T \widehat{\mathbf{\Sigma}} \v \geq \frac{\lambda_{\min}}{2}   ~,
\eeq
where $\cC$ is defined in \eqref{des_cone}.
\end{theo}
To prove Theorem \ref{kendall_re_sharp}, we first formally state below the convergence result for $\widehat{\mathbf{\Sigma}}$ and $\tilde{\mathbf{\Sigma}}$ in \cite{hali13}.
\begin{lemm}[Theorem 4.10 in \cite{hali13}]
\label{rate_res_spec}
Let $\bX = [\x_1, \x_2, \ldots, \x_n]^T$ be i.i.d. samples of $\x \sim TE(\tilde{\mathbf{\Sigma}}, \xi, \f)$ for which the sign sub-Gaussian condition holds with constant $\kappa$. With probability at least $1 - 2\alpha - \alpha^2$, $\widehat{\mathbf{\Sigma}}$ constructed from $\bX$ satisfies
\beq
\|\widehat{\mathbf{\Sigma}} - \tilde{\mathbf{\Sigma}}\|_{2,s_0} \leq \pi^2 \Bigg ( \frac{s_0 \log p}{n} +  2 \sqrt{2\kappa} \| \tilde{\mathbf{\Sigma}} \|_2 \sqrt{\frac{s_0\left(3 + \log (p/s_0)\right) + \log (1/\alpha)}{n}}\Bigg)~.
\eeq
\end{lemm}
The next step for showing Theorem \ref{kendall_re_sharp} is to extend the RE condition on all $s_0$-sparse unit vectors ($s_0$ needs to be appropriately specified) to all unit vectors inside the targeted descent cone $\cal C$. Lemma \ref{re_maurey} accomplishes this goal.
\begin{lemm}
\label{re_maurey}
Given $\widehat{\mathbf{\Sigma}}$ constructed from $\bX$ whose rows are generated from $\x \sim TE(\tilde{\mathbf{\Sigma}}, \xi, \f)$, we assume that for every $s_0$-sparse unit vector $\v$, the condition $\v^T \hat{\mathbf{\Sigma}} \v \geq \mu$ is satisfied. Then we have for any $\u \in \cC \cap \s^{p-1}$,
\beq
\u^T \widehat{\mathbf{\Sigma}} \u \geq \mu - \frac{4 s}{s_0 - 1} \left ( 1 - \mu \right )	~.
\eeq
\end{lemm}
\proof For any $\u \in \cC \cap \s^{p-1}$, let $\z \in \R^p$ be a random vector defined by
\beq
\P \left ( \z = \| \u\|_1 \sign(u_i) \cdot \e_i \right ) = \frac{|u_i|}{\|\u\|_1}	~,
\eeq
where $\{\e_i\}_{i=1}^p$ is the canonical basis of $\R^p$. Therefore, $\E[\z] = \u$. Let $\z_1, \z_2, \hdots, \z_{s_0}$ be independent copies of $\z$ and set $\bar{\z} = \frac{1}{s_0} \sum_{i=1}^{s_0} \z_i$. Therefore $\bar{\z}$ is an $s_0$-sparse vector, and by our assumption on quadratic forms on $s_0$-sparse vectors
\begin{align}
\bar{\z}^T \hat{\mathbf{\Sigma}} \bar{\z} \geq \mu \|\bar{\z}\|_2^2 \ \Longrightarrow \
\E\left[ \bar{\z}^T  \hat{\mathbf{\Sigma}} \bar{\z} \right] \geq \mu \E\left[\| \bar{\z} \|_2^2\right]~,
\label{eq:re_vc_dim2}
\end{align}
where the expectation is taken w.r.t $\bar{\z}$. Since $\bar{\z} = \frac{1}{s_0} \sum_{i=1}^{s_0} \z_i$, we have
\begin{align*}
\E\Big [ \bar{\z}^T  \hat{\mathbf{\Sigma}} \bar{\z} \Big] = \frac{1}{s_0^2} \sum_{1 \leq i,j \leq s_0} \E\left[ \z_i^T \hat{\mathbf{\Sigma}} \z_{j} \right] &= \frac{1}{s_0^2} \sum_{\substack{1 \leq i,j \leq s_0 \\ i \neq j}} \E\left[ \z_i^T \hat{\mathbf{\Sigma}} \z_{j} \right]
+ \frac{1}{s_0^2} \sum_{1 \leq i \leq s_0} \E \left[ \z_i^T \hat{\mathbf{\Sigma}} \z_{i} \right] \\
&= \frac{s_0(s_0 - 1)}{s_0^2} \u^T \hat{\mathbf{\Sigma}} \u + \frac{s_0}{s_0^2}  \sum_{i=1}^p \frac{|u_{i}|}{\|\u\|_1} \| \u\|_1^2 \hat{\sigma}_{ii} \\ 
&= \frac{s_0 - 1}{s_0} \u^T \hat{\mathbf{\Sigma}} \u + \frac{\|\u\|_1^2}{s_0} ~,
\end{align*}
since $\hat{\sigma}_{ii} = 1$, and $\sum_{i=1}^p \frac{|u_i|}{\| \u \|_1} = 1$.
Replacing $\hat{\mathbf{\Sigma}}$ in the above expression by the identity matrix $\mathbf{I} \in \R^{p \times p}$, we have
\begin{equation*}
\E\|\bar{\z}\|_2^2 = \frac{s_0-1}{s_0}  \|\u\|_2^2 + \frac{\|\u\|_1^2}{s_0}	~.
\end{equation*}
Plugging both these expressions back in \eqref{eq:re_vc_dim2}, we have
\begin{gather*}
\frac{s_0 - 1}{s_0} \u^T \hat{\mathbf{\Sigma}} \u + \frac{\|\u\|_1^2}{s_0} \geq \mu \frac{s_0-1}{s_0}  \|\u\|_2^2  + \mu \frac{\|\u\|_1^2}{s_0} \quad \Longrightarrow \\
 \u^T \hat{\mathbf{\Sigma}} \u  \geq \mu \| \u \|_2^2 - \frac{\| \u \|_1^2}{s_0 - 1} (1 - \mu)
\geq \mu - \frac{4s}{s_0 - 1} (1 - \mu) ~,
\end{gather*}
where we use the facts that $\|\u\|_2 = 1$ and $\|\u\|_1 \leq 2 \sqrt{s}$. That completes the proof. \qed

Note that Lemma \ref{re_maurey} is a deterministic result though the proof involves probabilistic argument. Equipped with Lemma \ref{rate_res_spec} and \ref{re_maurey}, we give the proof of Theorem \ref{kendall_re_sharp}.
\vspace{3mm}

\noindent{\itshape Proof of Theorem \ref{kendall_re_sharp}:}\hspace*{1em} For Lemma \ref{rate_res_spec}, we set $\alpha = \frac{1}{p}$, $s_0 = \frac{16s}{\lambda_{\min}}$, and let $c_0 = \max\{\frac{320 \kappa \pi^4 \|\tilde{\mathbf{\Sigma}}\|_2^2}{\lambda^2_{\min}}, \frac{\pi^2}{\lambda_{\min}}\}$. When $n \geq \frac{128 c_0}{\lambda_{\min}} s \log p = 8 c_0 s_0 \log p$, by Lemma \ref{rate_res_spec}, we have
\begin{align*}
 \|\widehat{\mathbf{\Sigma}} - \tilde{\mathbf{\Sigma}}\|_{2,s_0}  &\leq  \pi^2 \Bigg ( \frac{s_0 \log p}{n} + 2 \sqrt{2\kappa} \| \tilde{\mathbf{\Sigma}} \|_2 \sqrt{\frac{s_0(3 + \log (p/s_0)) + \log p}{n}} \Bigg ) \\
&\leq \pi^2 \Bigg ( \frac{s_0 \log p}{\frac{\pi^2}{\lambda_{\min}} \cdot 8 s_0 \log p} + 2 \sqrt{2\kappa} \| \tilde{\mathbf{\Sigma}} \|_2 \sqrt{\frac{s_0(3 + \log (p/s_0)) + \log p}{\frac{320 \kappa \pi^4  \|\tilde{\mathbf{\Sigma}}\|_2^2}{\lambda^2_{\min}} \cdot 8 s_0 \log p}} \Bigg ) \\
&\leq \pi^2 \left( \frac{\lambda_{\min}}{\pi^2} \sqrt{\frac{5 s_0 \log p}{320 s_0 \log p}} + \frac{\lambda_{\min}}{\pi^2} \frac{s_0 \log p}{8 s_0 \log p} \right) \\
&\leq \frac{\lambda_{\min}}{8} + \frac{\lambda_{\min}}{8} = \frac{\lambda_{\min}}{4} ~,
\end{align*}
with probability at least $1 - \frac{2}{p} - \frac{1}{p^2}$. It follows that for any $s_0$-sparse unit vector $\v$,
\begin{align*}
\v^T \widehat{\mathbf{\Sigma}} \v &\geq \v^T \tilde{\mathbf{\Sigma}} \v - \left| \v^T \left(\widehat{\mathbf{\Sigma}} - \tilde{\mathbf{\Sigma}}\right) \v \right| \geq \lambda_{\min} - \|\widehat{\mathbf{\Sigma}} - \tilde{\mathbf{\Sigma}}\|_{2, s_0} \geq \frac{3}{4} \lambda_{\min} ~,
\end{align*}
which satisfies the assumption in Lemma \ref{re_maurey} with $\mu = \frac{3}{4} \lambda_{\min}$. With the same $s_0 = \frac{16s}{\lambda_{\min}}$, by Lemma \ref{re_maurey}, we have for any $\v \in \cC \cap \s^{p-1}$,
\begin{align*}
\v^T \widehat{\mathbf{\Sigma}} \v & \geq \frac{3}{4} \lambda_{\min} - \frac{4 s}{\frac{16 s}{\lambda_{\min}} - 1} \left( 1 - \frac{3}{4} \lambda_{\min} \right)	 \\
&\geq \frac{3}{4} \lambda_{\min} - \frac{4 s}{\frac{16 s}{\lambda_{\min}} - 12 s} \left( 1 - \frac{3}{4} \lambda_{\min} \right) \\
&= \frac{3}{4} \lambda_{\min} - \frac{4 s}{\frac{16 s}{\lambda_{\min}} (1  - \frac{3}{4} \lambda_{\min})} \left( 1 - \frac{3}{4} \lambda_{\min} \right) \\
&= \frac{3}{4} \lambda_{\min} - \frac{\lambda_{\min}}{4} = \frac{\lambda_{\min}}{2} ~,
\end{align*}
which completes the proof. \qed

With RE condition satisfied by $\widehat{\mathbf{\Sigma}}$, we can proceed to the recovery guarantee for KDS. The next theorem relies on the RE condition described in Lemma \ref{kendall_re}, but we emphasize that if sign sub-Gaussian condition holds we can obtain similar result as long as $n$ attains the bound in Theorem \ref{kendall_re_sharp}, which is smaller than the one required in Lemma \ref{kendall_re}.
\begin{theo}
\label{theta_recovery}
For any $s$-sparse $\tilde{\Th}$, if we choose $\gamma_n = \frac{5 \pi}{\sqrt{\lambda_{\min}}} \sqrt{\frac{s \log p}{n }}$  and $n \geq \left(\frac{24 \pi}{\lambda_{\min}}\right)^2 s^2 \log p$
, with probability at least $1 - \frac{2}{p} - \frac{1}{p^{2.5}}$,  $\hat{\Th}$ given by \eqref{eq:kds} satisfies
\beq
\left\|\check{\Th} - \frac{\tilde{\Th}}{\sigma_y} \right\|_2 \leq \frac{40 \pi}{\lambda_{\min}^{3/2}} \sqrt{\frac{s^2 \log p}{n}} ~,
\eeq
\end{theo}
\begin{proof}
For the sake of convenience, we denote $\Th^* = \frac{\tilde{\Th}}{\sigma_y}$, and it is easy to see that $\tilde{\mathbf{\Sigma}} \Th^* = \boldsymbol{\beta}$.
We first show that $\Th^*$ is feasible when $\gamma_n = \frac{5 \pi}{\sqrt{\lambda_{\min}}} \sqrt{\frac{s \log p}{n}}$, by bounding the left-hand side of the constraint for $\Th^*$.
\begin{equation*}
\begin{split}
\Big \| \widehat{\mathbf{\Sigma}} \Th^* - \hat{\boldsymbol{\beta}} \Big \|_{\infty} &= \left\| \left( \widehat{\mathbf{\Sigma}} - \tilde{\mathbf{\Sigma}} \right) \Th^* - (\hat{\boldsymbol{\beta}} - \boldsymbol{\beta})  \right\|_{\infty} \\
&\leq \left\| \left( \widehat{\mathbf{\Sigma}} - \tilde{\mathbf{\Sigma}} \right) \Th^* \right\|_{\infty} + \left\| \hat{\boldsymbol{\beta}} - \boldsymbol{\beta}  \right\|_{\infty} \\
&\leq  \| \Th^* \|_1 \left\| \widehat{\mathbf{\Sigma}} - \tilde{\mathbf{\Sigma}}  \right\|_{\max} + 2 \pi \sqrt{\frac{\log p}{n}} \\
&\leq \sqrt{s} \cdot \| \Th^* \|_2  \left\| \widehat{\mathbf{\Sigma}} - \tilde{\mathbf{\Sigma}}  \right\|_{\max}  + 2 \pi \sqrt{\frac{\log p}{n}} \\
&\leq \frac{3 \pi}{\sqrt{\lambda_{\min}}} \sqrt{\frac{s \log p}{n }} + 2 \pi \sqrt{\frac{\log p}{n}} \leq \frac{5 \pi}{\sqrt{\lambda_{\min}}} \sqrt{\frac{s \log p}{n }} ~,
\end{split}
\end{equation*}
where we use Lemma \ref{lem:kendall_mat} and \ref{lem:kendall_vec}, and thus $\Th^*$ is feasible with probability $1 - \frac{2}{p} - \frac{1}{p^{2.5}}$ by union bound. On the other hand, since $\check{\Th}$ is optimal solution to \eqref{eq:kds}, it satisfies
\begin{align*}
 \|\check{\Th}\|_1 \leq \|\Th^*\|_1 \ \ \ \ \ \ \text{and} \ \ \ \ \ \ \left\| \widehat{\mathbf{\Sigma}} \check{\Th} - \hat{\boldsymbol{\beta}} \right\|_{\infty} \leq \gamma_n ~.
\end{align*}
Letting $\z = \check{\Th} - \Th^*$, we thus have
\begin{gather*}
\left\| \widehat{\mathbf{\Sigma}} \z \right\|_{\infty}  \leq \left\| \widehat{\mathbf{\Sigma}} \check{\Th} - \hat{\boldsymbol{\beta}} \right\|_{\infty} + \left\| \widehat{\mathbf{\Sigma}} \Th^* - \hat{\boldsymbol{\beta}} \right\|_{\infty} \leq 2 \gamma_n \ \ \ \Longrightarrow \\
\z^T  \widehat{\mathbf{\Sigma}} \z = \left\langle \z,  \widehat{\mathbf{\Sigma}} \z \right\rangle  \leq \|\z\|_1 \left\| \widehat{\mathbf{\Sigma}} \z \right\|_{\infty} \leq 2 \gamma_n \|\z\|_1
\end{gather*}
Using Lemma \ref{kendall_re} combined with the inequality above, with probability at least $1 - \frac{2}{p} - \frac{1}{p^{2.5}}$, we get
\begin{gather*}
\frac{\lambda_{\min}}{2} \|\z\|_2^2 \leq \z^T  \widehat{\mathbf{\Sigma}} \z \leq 2 \gamma_n \|\z\|_1 \ \ \ \ \Longrightarrow \ \ \ \ \|\z\|_2 \leq  \frac{4 \gamma_n}{\lambda_{\min}} \frac{\|\z\|_1}{\|\z\|_2} \leq  \frac{40 \pi}{\lambda_{\min}^{3/2}} \sqrt{\frac{s^2 \log p}{n}} ~,
\end{gather*}
where we use the fact that $\sup_{\z \in \cC} \frac{\|\z\|_1}{\|\z\|_2} \leq 2 \sqrt{s}$. \qed
\end{proof}

From the theorem above, though KDS only approximates a normalized version of $\tilde{\Th}$, the scale $\sigma_y$ can be estimated by computing the sample variance $\hat{\sigma}_y^2$ of $\y$, and the final estimate of $\tilde{\Th}$ is $\hat{\Th} = \hat{\sigma}_y \check{\Th}$ as shown in Algorithm \ref{alg:overview}. 

\subsection{Estimating $\cal F$}
After $\hat{\Th}$ is obtained, we can turn to the estimation of transformations $\cal F$. As we only have access to a finite number of samples $\{(\x_i, y_i)\}_{i=1}^n$, it is impossible to know the exact function. Hence we use the simple nearest-neighbor interpolation to approximate the $f_j$ as mentioned in \eqref{eq:interpolate}. By leveraging the monotonicity of $f_j$, we can estimate $\tilde{\bX}$ via solving the constrained least squares problem below,
\beq
\begin{gathered}
\label{lsq}
\hat{\bX} = \underset{\bZ \in \R^{n \times p}}{\argmin} \ \ \ell(\bZ) = \frac{1}{2} \|\bZ \hat{\Th} - \y\|_2^2 \ \ \ \ \text{s.t.} \ \ \ \ \z^j \in \cM(\x^j), \ \forall \ 1 \leq j \leq p ~,
\end{gathered}
\eeq
where the set $\cM(\x)$ denotes the \emph{monotone cone} induced by vector $\x$, i.e.,
\beq
\cM(\x) = \{ \v \ | \ v_i \geq v_j \ \text{iff} \ x_i \geq x_j, \ \forall \ 1 \leq i,j \leq p \} ~.
\eeq
The problem \eqref{lsq} is convex w.r.t. $\bZ$. Note that if $\hat{\Th} = \mathbf{1}$, the problem \eqref{lsq} is reduced to the estimation of $\cal F$ in AIM, which can be solved by the CPAV algorithm. Hence similar CPAV-type algorithm applies here, which is essentially a procedure of cyclic block coordinate descent (BCD) with exact minimization (i.e., minimizing $\ell(\bZ)$ w.r.t. each $\z^j$ cyclically while keeping other blocks fixed). In this scheme, each subproblem turns out to be an isotonic regression \cite{barl72}. To be specific, we let $\hat{\bX}_{(k)}$ be the iterate of the $k$-th round update, and define the residue for the $j$-th block as
\beq
\label{eq:residue}
\r^j_{(k)} = \y - \sum_{i < j} \hat{\theta}_i \hat{\x}^i_{(k)} - \sum_{i > j} \hat{\theta}_i \hat{\x}^i_{(k-1)} .
\eeq
Then each $\hat{\x}^j_{(k)}$ is obtained by solving
\beq
\label{lsqisoto}
\hat{\x}^j_{(k)} = \underset{\z^j \in \cM(\x^j)}{\argmin} \ \ \frac{1}{2} \Big \| \z^j   - \frac{\r^j_{(k)}}{\hat{\theta}_j} \Big \|_2^2 ~,
\eeq
which can be efficiently computed in $O(n)$ time using a skillful implementation of PAVA \cite{grwi84}. If we define for a set $\cA$ the projection operator as $P_\cA(\z) = \argmin_{\x \in \cA} \frac{1}{2} \|\x - \z\|_2^2$, the isotonic regression \eqref{lsqisoto} is simply the projection of $\r^j_{(k)} / \hat{\theta}_j$ onto the monotone cone $\cM(\x^j)$. Note that $\ell(\cdot)$ is a function of the design $\bZ$ instead of the coefficient vector $\hat{\Th}$. Though being convex, the problem \eqref{lsq} can have infinitely many solutions, some of which can be far from the original $\tilde{\bX}$. For example, given any $\hat{\bX}$, we can construct another optimum via shifting two columns $\hat{\x}^i$ and $\hat{\x}^j$ by $\mu_i$ and $\mu_j$ respectively, such that $\hat{\theta}_i \mu_i  + \hat{\theta}_j \mu_j = 0$. To avoid these ``bad'' solutions, we further impose on each $\hat{\x}^j$ the constraints $\mathbf{1}^T \hat{\x}^j = 0$ and $\|\hat{\x}^j\|_2 \leq \sqrt{n}$, as the marginal distribution of $\tilde{x}_{ij}$ is zero-mean and unit-variance. With additional constraints, the new problem is given by
\begin{gather}
\label{lsq_nrm}
\hat{\bX} \ = \ \underset{\bZ \in \R^{n\times p}}{\argmin} \ \ell(\bZ) \ \ \ \ \text{s.t.} \ \ \ \ \z^j \in \cM(\x^j), \ \mathbf{1}^T \z^j = 0, \ \|\z^j\|_2 \leq \sqrt{n}, \ \forall \ 1 \leq j \leq p ~, 
\end{gather}
and the subproblem for each block boils down to
\beq
\begin{gathered}
\label{lsqisoto_nrm}
\hat{\x}^j_{(k)} = \underset{\z^j \in \cM(\x^j)}{\argmin} \  \frac{1}{2} \Big \| \z^j   - \frac{\r^j_{(k)}}{\hat{\theta}_j} \Big \|_2^2  \ \ \ \ \text{s.t.} \ \ \ \  \mathbf{1}^T \z^j = 0, \  \|\z^j\|_2 \leq \sqrt{n} ,
\end{gathered}
\eeq
which we name \emph{standardized isotonic regression}. The solution to \eqref{lsqisoto_nrm} can be viewed as the projection onto the intersection of monotone cone $\cM(\x^i)$, hyperplane $\cL = \{ \z \ | \ \mathbf{1}^T \z = 0 \}$, and scaled $L_2$-norm ball $\cB = \{ \z \ | \ \|\z\|_2 \leq \sqrt{n} \}$. The next theorem show that the standardized isotonic regression is equivalent to the ordinary isotonic regression followed by successive projection on $\cL$ and $\cB$.
\begin{theo}
\label{sol_std_isoto}
Given any monotone cone $\cM$, the following equality holds
\beq
\label{proj}
P_{\cM \cap \cL \cap \cB}(\cdot) = P_{\cB}(P_{\cL}(P_{\cM}(\cdot))) ~,
\eeq
where $P_{\cL}(\z) = \z - \frac{\mathbf{1}^T \z}{n} \cdot \mathbf{1}$ and $P_{\cB}(\z) = \min \{ \frac{\sqrt{n}}{\|\z\|_2}, 1 \} \cdot \z$.
\end{theo}
\proof It is easy to verify the the analytic expression for $P_{\cL}(\cdot)$ and $P_{\cB}(\cdot)$. To show \eqref{proj}, we let $\x^* = P_{\cM}(\z)$ and $\tilde{\x}^* = P_{\cM \cap \cL \cap \cB}(\z)$. We assume w.l.o.g. that the monotone cone is $\cM = \{\x \ | \ x_1 \geq x_2 \geq \ldots \geq x_n \}$.  By introducing the Lagrange multipliers $\boldsymbol{\lambda} = [\lambda_1, \ldots, \lambda_{n-1}]^T$, the isotonic regression $P_{\cM}(\z)$ can be casted as
\begin{gather*}
\max_{\boldsymbol{\lambda} \preceq \mathbf{0}} \min_{\x} \ \ g(\x, \boldsymbol{\lambda}) = \frac{1}{2} \|\x - \z\|_2^2  + \sum_{i=1}^{n-1} \lambda_i (x_i - x_{i+1}) ~,
\end{gather*}
where we use the strong duality. The optimum $\x^*$ has to satisfy the stationarity $\nabla_{\x} \ g(\x, \boldsymbol{\lambda}) = 0$, i.e.,
\beq
\label{station_iso}
\begin{gathered}
x_1^* - z_1 + \lambda_1 = 0 ~,\\
x_2^* - z_2 - \lambda_1 + \lambda_2 = 0 ~,\\
\vdots \\
x_{n-1}^* - z_{n-1} - \lambda_{n-2} + \lambda_{n-1} = 0 ~,\\
x_n^* - z_n - \lambda_{n-1} = 0 ~.
\end{gathered}
\eeq
Using \eqref{station_iso} to express $\x^*$ in terms of $\boldsymbol{\lambda}$, we denote $\min_{\x} g(\x, \boldsymbol{\lambda})$ by another function $h(\boldsymbol{\lambda})$, and the optimal dual variables $\boldsymbol{\lambda}^*$ satisfies
\begin{align*}
\boldsymbol{\lambda}^* = \underset{\boldsymbol{\lambda} \preceq \mathbf{0}}{\argmax} \ \ h(\boldsymbol{\lambda}) ~.
\end{align*}
For the standardized isotonic regression $P_{\cM \cap \cL \cap \cB}(\z)$, we can also introduce the Lagrange multipliers $\boldsymbol{\lambda} = [\lambda_1, \ldots, \lambda_{n-1}]^T$, $\beta$ and $\gamma$, and obtain the following optimization problem
\beq
\max_{\lambda \preceq \mathbf{0}, \gamma \leq 0, \beta} \min_{\x} \ \ \tilde{g}(\x, \boldsymbol{\lambda}, \beta, \gamma) = \frac{1}{2} \|\x - \z\|_2^2 + \sum_{i=1}^{n-1} \lambda_i (x_i - x_{i+1}) + \beta \sum_{i=1}^n x_i + \gamma (n - \|\x\|_2^2) ~.
\eeq
Again the optimum $\tilde{\x}^*$ has to satisfy $\nabla_{\x} \ \tilde{g}(\tilde{\x}^*, \boldsymbol{\lambda}, \beta, \gamma)$,
\beq
\label{station_iso_nrm}
\begin{gathered}
(1 - 2 \gamma) \tilde{x}_1^* - z_1 + \beta + \lambda_1 = 0 ~,\\
(1 - 2 \gamma) \tilde{x}_2^* - z_2 + \beta - \lambda_1 + \lambda_2 = 0 ~,\\
\vdots \\
(1 - 2 \gamma) \tilde{x}_{n-1}^* - z_{n-1} + \beta  - \lambda_{n-2} + \lambda_{n-1} = 0 ~,\\
(1 - 2 \gamma) \tilde{x}_n^* - z_n + \beta - \lambda_{n-1} = 0 ~.
\end{gathered}
\eeq
By substituting $\tilde{\x}^*$ for $\boldsymbol{\lambda}$, $\beta$ and $\gamma$, we have
\begin{align*}
\min_{\x} \tilde{g}(\x, \boldsymbol{\lambda}, \beta, \gamma) &= \frac{1 - 2 \gamma}{2} \sum_{i=1}^n \left( \tilde{x}_i^* - \frac{z_i - \beta}{1-2\gamma} \right)^2 + \sum_{i=1}^{n-1} \lambda_i (\tilde{x}_i^* - \tilde{x}_{i+1}^*) + \frac{\|\z\|_2^2}{2}  - \frac{\sum_{i=1}^n (z_i - \beta)^2}{2(1 - 2 \gamma)} + \gamma n \\
&= \frac{h(\boldsymbol{\lambda})}{1 - 2 \gamma} + \frac{\|\z\|_2^2}{2}  - \frac{\sum_{i=1}^n (z_i - \beta)^2}{2(1 - 2 \gamma)} + \gamma n ~, 
\end{align*}
in which we note that the last three terms are free of $\boldsymbol{\lambda}$. Hence the optimal $\boldsymbol{\lambda}$ for standardized isotonic regression,
\begin{align*}
\boldsymbol{\lambda}^* &= \underset{\boldsymbol{\lambda} \preceq \mathbf{0}}{\argmax} \ \ \frac{h(\boldsymbol{\lambda})}{1 - 2 \gamma} + \frac{\|\z\|_2^2}{2}  - \frac{\sum_{i=1}^n (z_i - \beta)^2}{2(1 - 2 \gamma)} + \gamma n  \\
&= \underset{\boldsymbol{\lambda}  \preceq \mathbf{0}}{\argmax} \ \ h(\boldsymbol{\lambda})
\end{align*}
is the same as the one for isotonic regression. Thus, combining \eqref{station_iso} and \eqref{station_iso_nrm}, we have
\beq
\label{sol_iso_nrm}
\tilde{\x}^* = \frac{\x^* - \beta \cdot \mathbf{1}}{1 - 2 \gamma} ~.
\eeq
On the other hand, by summing up the equations respectively in \eqref{station_iso} and \eqref{station_iso_nrm} and the primal feasibility $\sum_{i=1}^n \tilde{x}^*_i = 0$ we have
\begin{align*}
\sum_{i=1}^n x^*_i = \sum_{i=1}^n z_i, \ \  \sum_{i=1}^n z_i =  n \beta \ \ \ \Longrightarrow \ \ \ \beta = \frac{\mathbf{1}^T {\x^*} }{n} ~,
\end{align*}
which implies that
\beq
\x^* - \beta \cdot \mathbf{1} = P_{\cL} (\x^*) = P_{\cL}(P_{\cM}(\z)) ~.
\eeq
Denoting $\x^* - \beta \cdot \mathbf{1}$ by $\hat{\x}^*$, we now show that scaling $\hat{\x}^*$ by $\frac{1}{1- 2\gamma}$ is exactly the projection onto $\cB$.
If $\|\hat{\x}^*\|_2 > \sqrt{n}$, then $\gamma < 0$ due to \eqref{sol_iso_nrm} and primal feasibility $\|\tilde{\x}^*\|_2 \leq \sqrt{n}$. By complementary slackness $\gamma (n - \|\tilde{\x}^*\|_2^2) = 0$, we have $\|\tilde{\x}^*\|_2 = \sqrt{n}$. If $\|\hat{\x}^*\|_2 < \sqrt{n}$, then $\|\tilde{\x}^*\| < \sqrt{n}$ due to \eqref{sol_iso_nrm} and dual feasibility $\gamma \leq 0$. It follows from complementary slackness that $\gamma = 0$, which result in $\tilde{\x}^* = \hat{\x}^*$. If $\|\hat{\x}^*\|_2 = \sqrt{n}$, by similar argument, we have $\tilde{\x}^* = \hat{\x}^*$ as well. In a word, we have
\begin{align*}
\tilde{\x}^* = \left\{
             \begin{array}{lll}
              \hat{\x}^*, &\text{if \ $\|\hat{\x}^*\|_2 \leq \sqrt{n}$} \\
               \frac{\sqrt{n} }{\|\hat{\x}^*\|_2} \hat{\x}^*,  &\text{if \ $\|\hat{\x}^*\|_2 > \sqrt{n}$}
             \end{array}  \right. ~,
\end{align*}
which matches the expression for $P_{\cB}(\cdot)$. We complete the proof by noting $\tilde{\x}^* = P_{\cB}(\hat{\x}^*) = P_{\cB}(P_{\cL}(P_{\cM}(\z)))$. \qed

Theorem \ref{sol_std_isoto} indicates that the extra cost for each subproblem of our CPAV algorithm is very minimal, since the projection onto $\cal L$ and $\cal B$ can be done in linear time. Note that the CPAV for AIM needs to work with $p$ blocks of variables, and pre-specifying the monotonicity for each $f_j$ could lead to as many as $2^p$ different combinations, which is computationally prohibitive. In contrast, our algorithm only deals with roughly $O(s)$ blocks and need not specify the monotonicity. The details of our CPAV is given Algorithm \ref{alg:cpav}. For $\hat{\theta}_j = 0$, the corresponding $f_j$ will have no contribution to the estimated SLIM, which is thus skipped in our CPAV.
The convergence of Algorithm \ref{alg:cpav} basically follows from the extensive studies on cyclic BCD type algorithms \cite{luts92,tsen01,bete13}. Recently \cite{suho15} show that the convergence rate of BCD with exact minimization achieves $O(1 / t)$ for a family of quadratic nonsmooth problem without linear dependency on the number of blocks, which applies to Algorithm \ref{alg:cpav} for solving \eqref{lsq_nrm}.
\begin{algorithm}[h]
\renewcommand{\algorithmicrequire}{\textbf{Input:}}
\renewcommand{\algorithmicensure} {\textbf{Output:}}
\caption{Estimating $\tilde{\bX}$}
\label{alg:cpav}
\begin{algorithmic}[1]
\REQUIRE Data $\y \in \R^n$, $\bX \in \R^{n \times p}$, estimated $\hat{\Th}$, number of round $t$  \ \\
\ENSURE Estimated hidden design $\hat{\bX}$ \\
\STATE Initialize $\hat{\bX}_{(0)} = \mathbf{0}_{n\times p}$
\FOR {$k$:= $1, 2, \ldots, t$}
\FOR {$j$:= $1, 2, \ldots, p$}
\IF {$\hat{\theta}_j \neq 0$}
\STATE Compute $\r^j_{(k)}$ using \eqref{eq:residue}
\STATE Compute $\z^j_{(k)}$ = $P_{\cM(\x^j)}\left(\frac{\r^j_{(k)}}{\hat{\theta}_j}\right)$ using PAVA
\STATE $\hat{\x}^j_{(k)}$ := $P_{\cB}(P_{\cL}(\z^j_{(k)}))$
\ENDIF
\ENDFOR
\ENDFOR
\STATE \textbf{Return} \ $\hat{\bX} = \hat{\bX}_{(t)}$
\end{algorithmic}
\end{algorithm}

\section{Experimental Results}
\label{sec:exp}
In this section, we show some experimental evidence for the effectiveness of SLIM. We test our estimation algorithm on the synthetic data. Specifically we fix the problem dimension $p = 500$, the sparsity level of $\tilde{\Th}$, $s = 10$. The distribution of $\x$ is chosen as $NPN(\tilde{\mathbf{\Sigma}}, \f)$, and $y \sim \langle \tilde{\Th}, \tilde{\x} \rangle + \cN(0,0.25)$. The covariance matrix is given by $\tilde{\mathbf{\Sigma}} = \bA \bA^T$, where $\bA$ is a Gaussian random matrix with normalized rows. In data preparation, we first generate $\tilde{\x}$ from $\cN(\mathbf{0}, \tilde{\mathbf{\Sigma}})$. For the ten $\tilde{x}_j$'s whose corresponding $\tilde{\theta}_j$'s are nonzero, we then apply ten different monotonically increasing functions to obtain $x_j$'s, which are basically the inverse of $f_j$'s. The ten inverse functions are summarized in the table below.
\begin{table}[h]
\centering
\begin{tabular}{|c|c|}
  \hline
  $f_1^{-1}(x) = x^3$ & $f_6^{-1}(x) = x\log (|x|+1)$ \\
  \hline
   $f_2^{-1}(x) = \sign(x) \sqrt{|x|}$  & $f_7^{-1}(x) = 1 / (1 + \exp(-x))$  \\
  \hline
  $f_3^{-1}(x) = \exp(x)$  & $f_8^{-1}(x) = x - 1$  \\
  \hline
  $f_4^{-1}(x) = \Phi(x)$  & $f_9^{-1}(x) = \sign(x)  \log(|x|+1)$ \\
  \hline
   $f_{5}^{-1}(x) = x \exp(\sqrt{|x|})$  & $f_{10}^{-1}(x) = \log(\exp(x)+1)$ \\
  \hline
\end{tabular}
\vspace{-1mm}
\caption{Inverse of $f_j$ for nonzero $\tilde{\theta}_j$}
\end{table}
The $\Phi(\cdot)$ in $f_4^{-1}$ is the CDF of standard norm distribution. For the rest of $\tilde{x}_j$, we randomly apply one of the functions above. All the results are obtained based on the average over 100 trials.

We plot in Figure \ref{fig:1} the normalized estimation error of $\tilde{\Th}$ and $\tilde{\bX}$, $\frac{\|\tilde{\Th} - \hat{\Th}\|_2}{\|\tilde{\Th}\|_2}$ and $\frac{\|\tilde{\bX} - \hat{\bX}\|_2}{\|\tilde{\bX}\|_2}$. As sample size $n$ increases from 100 to 500, we can see the clear decreasing trend of error. We also compare the prediction error of SLIM with the simple linear model on 200 new data points, which is shown in Figure \ref{fig:2}. The best tuning parameters for both methods are picked up via grid search. The simple linear model fails to capture the nonlinear correlation between $\x$ and $y$, thus incurring large prediction errors. In contrast, SLIM better fits the data and has substantially smaller errors. In Figure \ref{fig:3}, we specifically plot the prediction errors along the parameter-tuning paths when $n = 500$, and see that SLIM always outperforms the linear model (The actual parameters are different for both methods, but we keep the largest as $2^9$ times the smallest). In Figure \ref{func_rec}, we also provide the plots for $f_1, f_2, \cdot, f_{10}$ and the corresponding estimated ones at the observed $x_1, x_2, \ldots, x_{10}$. It is not difficult to see that the red dots well capture the shape of the function plots except for some tails.
\begin{figure*}[hbt!]
\centering
\subfigure[Estimation error vs. sample size]{
\label{fig:1}
    \includegraphics[width=0.312\textwidth]{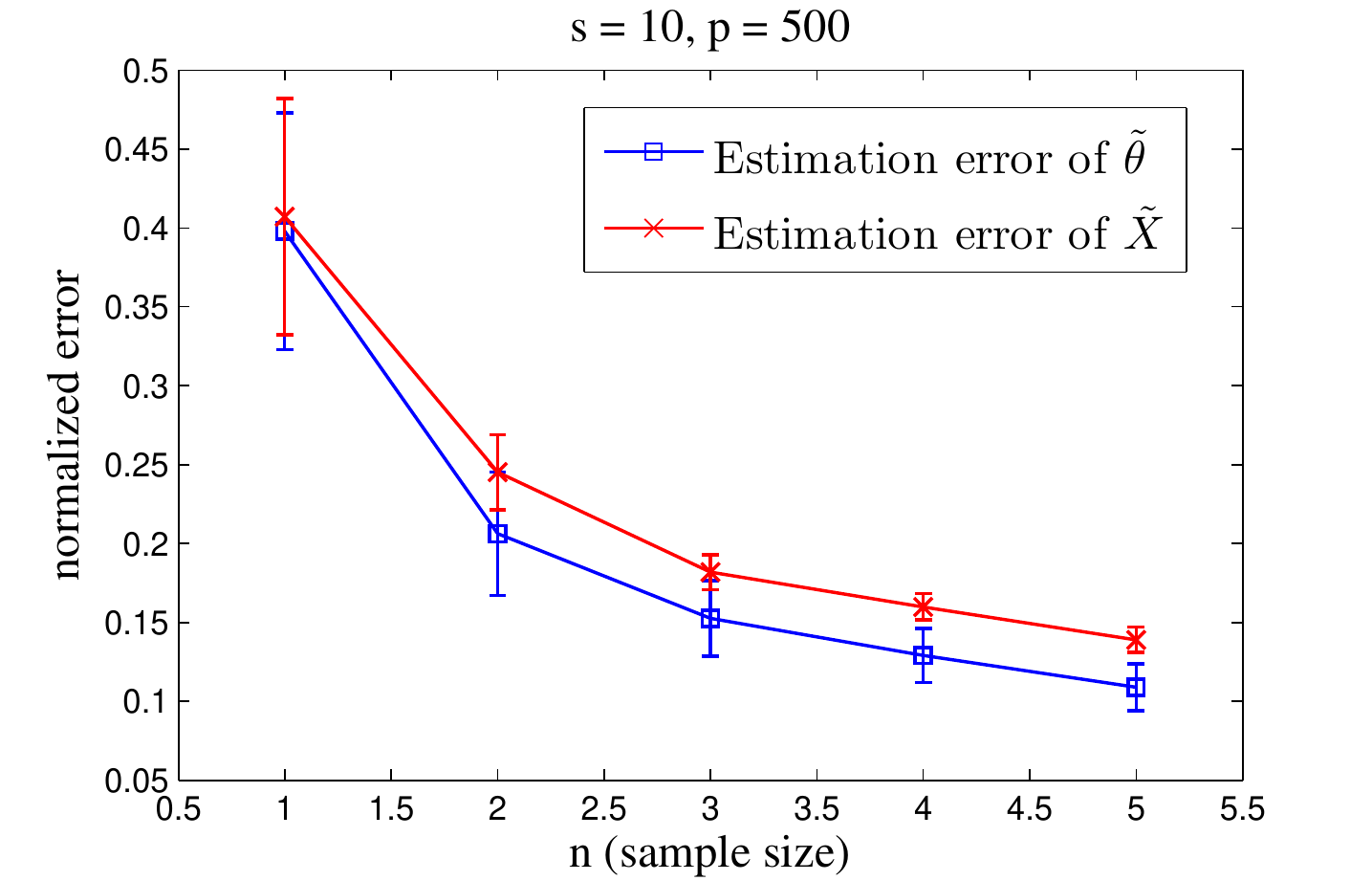}}%
\hspace{.1in}
\subfigure[Prediction error vs. sample size]{
\label{fig:2}
    \includegraphics[width=0.312\textwidth]{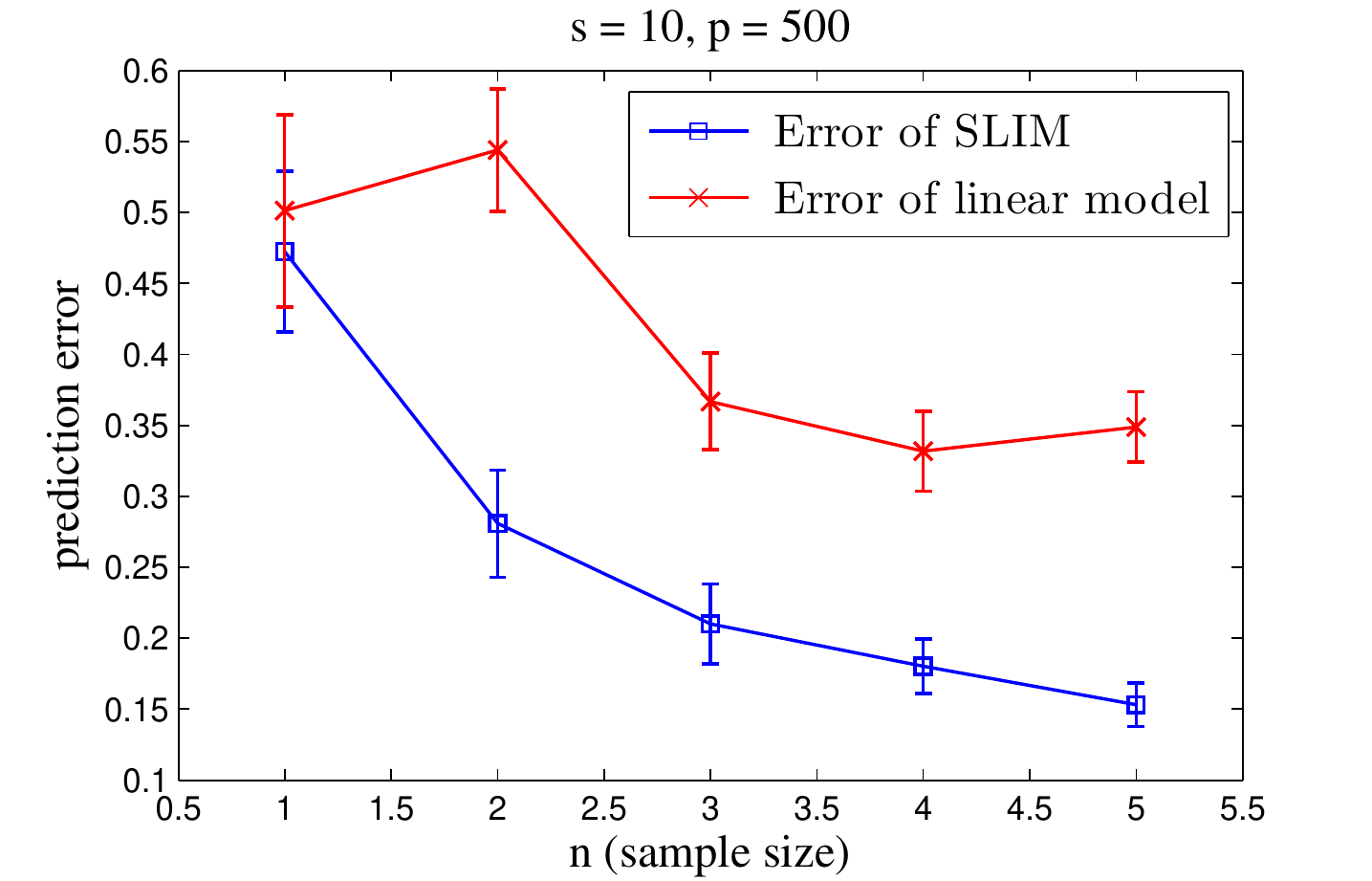}}%
\hspace{.1in}
\subfigure[Prediction error vs. tuning parameter]{
\label{fig:3}
    \includegraphics[width=0.312\textwidth]{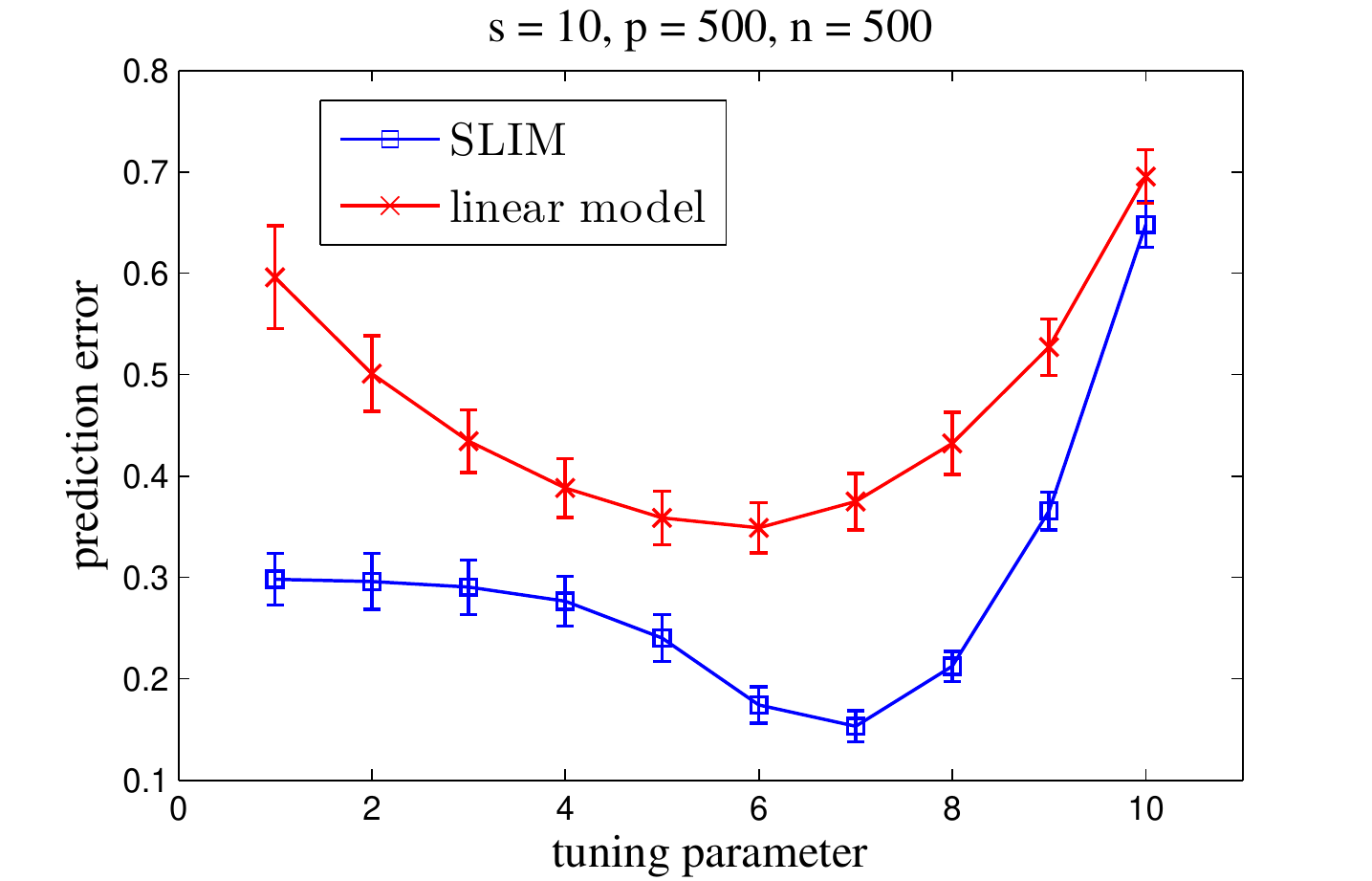}}%
\vspace{-3mm}
\caption{Experimental results for SLIM}
\label{error_vs_n_sub}
\vspace*{-3mm}
\end{figure*}
\begin{figure*}[hbt!]
\centering
\hspace*{-2cm}	
\includegraphics[width=1.21\linewidth]{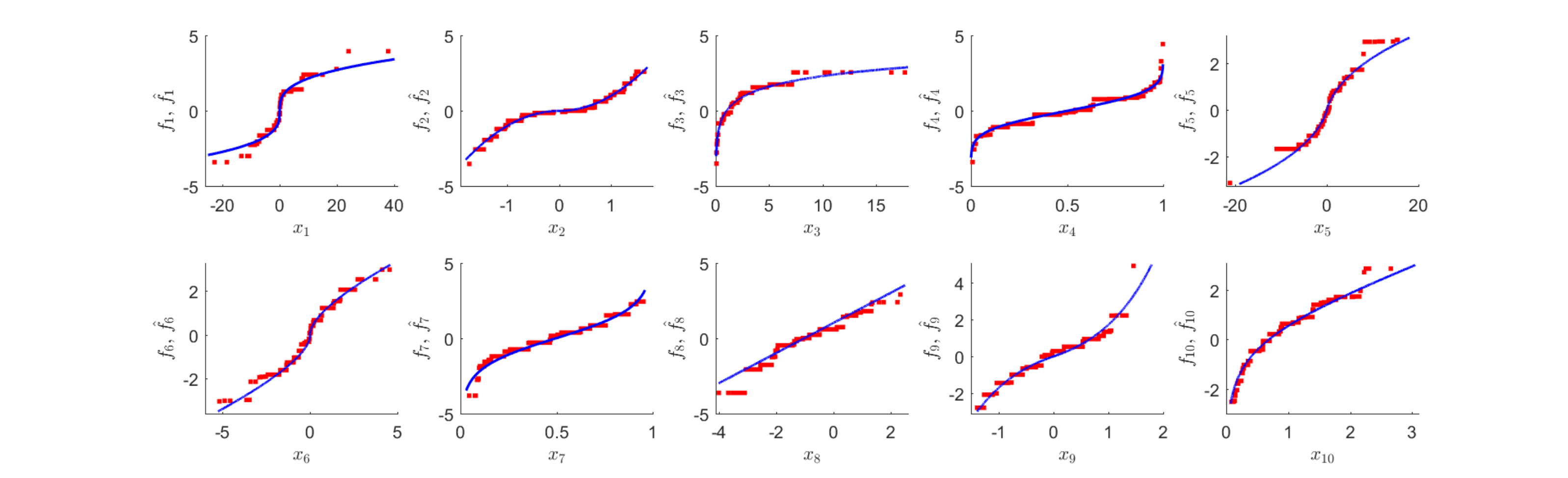}
    \vspace{-8mm}
	\caption{Function $f_j$ (blue curves) and the corresponding $\hat{f}_j$ at observed $x_j$ (red dots) ($n=500$)}
	 \label{func_rec}
\end{figure*}
\vspace{-2mm}

\section{Conclusions}
\label{sec:conc}
In this paper, we propose the sparse linear isotonic models (SLIMs) together with a two-step estimation algorithm, which aims to uncover the underlying linear models, when we only get access to the monotonically transformed values of the predictors. Our model enjoys a few advantages over the classical additive isotonic models (AIMs). In high-dimensional setting, the proposed Kendall's tau Dantzig selector can provably recover the sparse parameters under suitable statistical assumptions. Especially we can obtain a sharper sample complexity than previous analysis when the so-called sign sub-Gaussian condition holds. On the optimization side, we show that as the subproblem in our backfitting algorithm for estimating monotone transformations, the standardized isotonic regression can be solved as efficiently as the ordinary isotonic regression. Some empirical evidences also demonstrate the effectiveness of SLIM.  

\vspace*{3mm}
{\bf Acknowledgements:} The research was supported by NSF grants IIS-1563950, IIS-1447566, IIS-1447574, IIS-1422557, CCF-1451986, CNS- 1314560, IIS-0953274, IIS-1029711, NASA grant NNX12AQ39A, and gifts from Adobe, IBM, and Yahoo.

\bibliography{ref}
\bibliographystyle{plain}

\end{document}